\documentclass{article}

\PassOptionsToPackage{numbers, compress}{natbib}
 \usepackage[preprint]{neurips_2026}

\usepackage[utf8]{inputenc} 
\usepackage[T1]{fontenc}    
\usepackage{hyperref}       
\usepackage{url}            
\usepackage{booktabs}       
\usepackage{amsfonts}       
\usepackage{nicefrac}       
\usepackage{microtype}      
\usepackage{xcolor}         
\usepackage{amsmath} 
\usepackage{amssymb}
\usepackage{graphicx}
\usepackage{wrapfig}
\usepackage{tcolorbox}    
\usepackage{caption}
\usepackage{colortbl}
\usepackage{enumitem}
\usepackage{lineno}
\usepackage{algorithm}
\usepackage{algpseudocode}

\usepackage{amsthm}

\definecolor{rowgreen}{HTML}{F1F8F2}
\definecolor{upgreen}{HTML}{2E7D32}
\definecolor{downorg}{HTML}{D4A017}
\definecolor{rowblue}{HTML}{E8F0FE}

\newcommand{\oursrowa}{\rowcolor{rowgreen}}
\newcommand{\oursrowb}{\rowcolor{rowblue}}
\newcommand{\inc}[1]{{\scriptsize\textcolor{upgreen}{$_{\uparrow#1}$}}}

\definecolor{phaseblue}{RGB}{30,100,180}
\newcommand{\stepcomment}[1]{\textcolor{phaseblue}{\textit{#1}}}

\usepackage{tcolorbox}
\tcbuselibrary{breakable}

\newcommand{\n}{\textsc{OPPO}}

\newtheorem{proposition}{Proposition}

\newtcolorbox{modelingchoice}[1][]{
  colback=blue!2,
  colframe=blue!2,
  boxrule=0pt,
  arc=2pt,
  outer arc=2pt,
  top=4pt, bottom=4pt,
  left=8pt, right=8pt,
  fonttitle=\small\sffamily\bfseries,
  coltitle=blue!50!black,
  title={#1},
  breakable
}

\title{OPPO: Bayesian Value Recursion for Token-Level Credit Assignment in LLM Reasoning}

\author{%
  Yu Li$^{1}$ \quad Rui Miao$^{2}$ \quad Tian Lan$^{1}$ \quad Zhengling Qi$^{1}$ \\
  $^{1}$George Washington University \quad $^{2}$The University of Texas at Dallas
}

\begin{document}

\maketitle

\begin{abstract}
Reinforcement learning with verifiable rewards has become the standard recipe for improving LLM reasoning, but the dominant algorithm GRPO assigns a single trajectory-level advantage to every token, diluting the signal at pivotal reasoning steps and injecting noise at uninformative ones. Critic-free alternatives derived from on-policy distillation supply per-token signals through oracle-conditioned likelihood ratios, yet apply each signal in isolation from the trajectory-level evidence accumulated up to that position. We propose Oracle-Prompted Policy Optimization (\n), which rests on a single observation: the oracle signal used by prior distillation-style methods for local discrimination is also the natural Bayesian update of the model's belief about eventual success. Accumulating the signal along a trajectory yields, in closed form and at the cost of one extra forward pass, a running estimate of the success probability at every position, together with a token-level advantage that requires no learned value network and no additional rollouts. A first-order analysis factorizes the advantage into the per-token discrimination signal used by distillation methods modulated by a state weight that concentrates credit on genuinely pivotal tokens, with a directional variance-reduction guarantee. The framework admits two estimators differing only in which model scores the evidence: a \textit{self-oracle} that reuses the student and recovers the on-policy distillation reward as a strict special case, and a \textit{teacher-oracle} that delegates scoring to a stronger frozen model. On two base LLMs across seven mathematics, science, and code reasoning benchmarks, \n\ improves over GRPO, DAPO, and SDPO by up to $+6.0$ points on AMC'23 and $+5.2$ points on AIME'24, with gains that widen monotonically with response length.
\end{abstract}
\vspace{-2mm}
\section{Introduction}
\label{sec:intro}
\vspace{-2mm}
Reinforcement learning with verifiable rewards (RLVR) has become a central paradigm for improving the reasoning capabilities of large language models~\citep{guo2025deepseek,yang2025qwen3,jaech2024openai}. In the standard pipeline, a model generates a complete reasoning trajectory, receives a binary reward based on the correctness of the final answer, and updates its policy to increase the likelihood of generating correct trajectories. Group Relative Policy Optimization (GRPO)~\citep{guo2025deepseek} is the most widely adopted on-policy algorithm in the paradigm, estimating the advantage of each trajectory relative to a group of sampled responses and feeding the signal into a clipped policy gradient update.

GRPO broadcasts a single scalar advantage to every token position~\citep{liu2025understanding}. In a reasoning chain of several hundred tokens, only a small subset of tokens typically determine the outcome, namely those corresponding to a strategy selection, a key algebraic step, or a critical logical deduction~\citep{patil2025advancing,sun2025survey}. By assigning equal credit to every position, GRPO dilutes the learning signal at pivotal tokens and injects noise through undeserved gradient updates at uninformative ones, and the problem worsens as reasoning chains grow longer~\citep{yu2025dapo,jain2025towards}.

Recent work has pursued two strategies. The first refines the GRPO framework while retaining the trajectory-level advantage. Dr.GRPO~\citep{liu2025understanding}, DAPO~\citep{yu2025dapo}, GSPO~\citep{zheng2025group}, and GMPO~\citep{zhao2025geometric} modify the importance ratio, clipping, or advantage normalization, yet every token still receives the same advantage~\citep{guo2025segment}. The second seeks finer-grained credit at the token level. Monte Carlo resampling methods such as VinePPO~\citep{kazemnejad2024vineppo} and TreeRL~\citep{hou2025treerl} require orders of magnitude more inference, implicit process reward models such as PRIME~\citep{cui2025process} require online updating, and segment-level aggregation in SPO~\citep{guo2025segment} pools descendant outcomes within shared prefixes. Most relevant to the present work, on-policy distillation methods with privileged information~\citep{agarwal2024policy,gu2023minillm,lu2025onpolicydistillation,yang2026self,hubotter2026reinforcement} condition the model on the ground-truth answer $y^\ast$ and use the resulting log-likelihood ratio $\log\lambda_t = \log\pi_{\mathrm{teacher}}(y_t \mid s_t) - \log\pi_{\mathrm{student}}(y_t \mid s_t)$ as a token-level credit signal, where $\log\lambda_t$ is the one-step expert log-ratio between the success-conditional optimal policy $\pi^\ast$ and the student. Every such method, however, applies $\log\lambda_t$ with uniform weight across positions and treats each token's credit as an isolated one-step quantity decoupled from the cumulative evidence along $y_{<t}$.

Motivated by this gap, we propose Oracle-Prompted Policy Optimization (\n), which rests on a single observation: the per-token oracle signal used in distillation-style methods is also the natural Bayesian update of the model's belief about eventual success. Accumulating the signal along a trajectory yields, in closed form and at no extra rollout cost, a running estimate of the success probability at every position, and from it a token-level advantage that requires no learned value network and sums exactly to the trajectory's outcome relative to its prior. The recursion admits two estimators that differ only in which model scores the evidence: a self-oracle reuses the student in one extra forward pass and recovers the on-policy distillation reward as a strict special case, while a teacher-oracle delegates scoring to a stronger frozen model. The advantage factorizes as the per-token discrimination signal modulated by a state weight $V_t(1-V_t)$ that peaks where the running belief is most uncertain and vanishes once the trajectory has committed, concentrating credit on pivotal tokens without any tunable hyperparameter. A direction-anchoring step ties the sign of each token-level advantage to the sequence-level outcome, preserving the verifier as the sole arbiter of which trajectories are reinforced while letting the Bayesian machinery redistribute credit within each trajectory.

Our contributions are as follows.
\begin{itemize}[leftmargin=*,topsep=2pt,itemsep=2pt,parsep=0pt]
    \item Showing that the OPD per-token reward $\log\lambda_t$ is simultaneously an expert log-ratio surrogate for $A_t^{\pi^\ast}$ and the Bayesian sufficient statistic for $V_t$, with the choice of probability measure cleanly selecting the self-oracle or teacher-oracle estimator.
    \item Deriving a closed-form Bayesian value recursion that produces a token-level advantage with no learned critic and no extra rollouts, satisfies the telescoping budget $\sum_t A_t = R - V_0$, and admits the factorization $A_t \approx V_t(1-V_t)\log\lambda_t$ with a directional variance-reduction guarantee.
    \item Integrating the recursion into the GRPO pipeline via direction anchoring and evidence clipping at the cost of one extra forward pass per trajectory, with identical downstream computation under either oracle.
    \item Evaluating the method on Qwen3 and Phi reasoning models across seven reasoning benchmarks, where the proposed framework consistently surpasses representative trajectory-level and oracle-conditioned baselines, with up to $+5.7$ points on AMC'23 and $+5.2$ points on AIME'24 under Teacher-\n\ on DAPO and gains that widen monotonically with response length.
\end{itemize}
\vspace{-3mm}
\section{Preliminaries}\label{sec:prelim}

\subsection{Reinforcement Learning for LLM Reasoning}
We model autoregressive LLM generation as a finite-horizon Markov decision process $(\mathcal{S}, \mathcal{V}, \pi_\theta, R)$ over a finite vocabulary $\mathcal{V}$. Given a query $x$, the state at step $t$ is the prefix $s_t \triangleq (x, y_{<t})$ with $y_{<t} \triangleq (y_1, \dots, y_{t-1})$ and $s_1 = (x)$. The policy $\pi_\theta(\cdot \mid s_t) \in \Delta(\mathcal{V})$ samples a token $y_t$, yielding the deterministic transition $s_{t+1} = (s_t, y_t)$. Generation terminates at the first step $T \leq T_{\max}$ where $y_T$ is the end-of-sequence symbol. A verifier assigns a terminal binary reward $R(\tau) \in \{0,1\}$ indicating answer correctness, with all intermediate rewards equal to zero. Training maximizes the expected return through policy gradient
\begin{equation}\label{eq:objective}
    J(\theta) \triangleq \mathbb{E}_{x \sim \mathcal{D},\,\tau \sim \pi_\theta(\cdot \mid x)}\!\left[R(\tau)\right], \qquad
    \nabla_\theta J(\theta) = \mathbb{E}\!\left[\sum_{t=1}^{T} A_t(s_t, y_t)\,\nabla_\theta \log \pi_\theta(y_t \mid s_t)\right],
\end{equation}
where $\mathcal{D}$ is the query distribution and $A_t : \mathcal{S} \times \mathcal{V} \to \mathbb{R}$ is any per-token advantage estimator unbiased against $Q_t(s_t, y_t) \triangleq \mathbb{E}_{\pi_\theta}[R \mid s_t, y_t]$ when paired with the score function $\nabla_\theta \log \pi_\theta(y_t \mid s_t)$.

GRPO~\citep{guo2025deepseek} forgoes a learned value baseline by drawing a group of $G \geq 2$ trajectories $\{\tau_i\}_{i=1}^{G} \sim \pi_\theta(\cdot \mid x)$ per query and assigning every token in $\tau_i$ the trajectory-level standardized advantage
\begin{equation}\label{eq:grpo}
    \hat{A}_i^{\mathrm{GRPO}} \triangleq \frac{R_i - \mu}{\sigma + \varepsilon}, \qquad \mu \triangleq \frac{1}{G}\sum_{j=1}^{G} R_j, \quad \sigma \triangleq \sqrt{\frac{1}{G}\sum_{j=1}^{G}(R_j - \mu)^2},
\end{equation}
so that $A_t(\tau_i) \equiv \hat{A}_i^{\mathrm{GRPO}}$ for all $t \in \{1,\dots,T_i\}$. The estimator is exchangeable in $t$, so if a subset $\mathcal{C}(\tau)$ of cardinality $k \ll T$ carries the causal responsibility for $R$, the signal at $t \in \mathcal{C}(\tau)$ is attenuated by a factor of $k/T$ while the remaining $T - k$ positions receive gradient mass with zero expected contribution~\citep{jain2025learning,li2026reason}.

The cleanest reference quantity for token-level credit is the on-policy state value
\begin{equation}\label{eq:value}
    V_t \triangleq P^{\pi_\theta}(R = 1 \mid x, y_{1:t-1}), \qquad V_0 \triangleq P^{\pi_\theta}(R{=}1 \mid x), \qquad V_{T+1} \triangleq R,
\end{equation}
which satisfies the Bellman consistency $V_t = \mathbb{E}_{y_t \sim \pi_\theta}[V_{t+1}]$. The ideal token-level advantage is the one-step temporal difference $A_t \triangleq V_{t+1} - V_t = Q_t(s_t, y_t) - V_t(s_t)$, and telescoping along the trajectory gives
\begin{equation}\label{eq:ideal_adv}
    \sum_{t=1}^{T} A_t \;=\; V_{T+1} - V_0 \;=\; R - V_0,
\end{equation}
so the question is how to distribute the credit budget $R - V_0$ across the $T$ positions. PPO~\citep{schulman2017proximal} trains a value network $V_\phi(s_t)$ to approximate Eq.~\eqref{eq:value} at substantial memory cost; GRPO sets $V_t \equiv V_0$ and distributes $R - V_0$ uniformly.
Note that token-level estimators concentrate the learning signal on the subset of positions causally responsible for $R$, while estimators constant in $t$ inject the variance of $R$ uniformly across all positions~\citep{liu2025uniform,song2026survey}.

\vspace{-1mm}
\subsection{On-Policy Distillation}\label{sec:opd}
\vspace{-1mm}

On-policy distillation (OPD) pairs a student policy $\pi_\theta$ with a frozen teacher $\pi_T$ and supervises $\pi_\theta$ on its own rollouts under a per-token divergence~\citep{agarwal2024policy,gu2023minillm,lu2025onpolicydistillation}. In the reasoning setting most relevant to RLVR, the teacher is conditioned on privileged information such as the ground-truth answer $y^\ast$, so that $\pi_T(\cdot \mid x, y_{<t}) \equiv \pi_T(\cdot \mid x, y^\ast, y_{<t})$ approximates the success-conditional next-token distribution $P(\cdot \mid x, y_{<t}, R{=}1)$~\citep{zhao2026self,hubotter2026reinforcement}. Drawing $y_{1:T} \sim \pi_\theta(\cdot \mid x)$, OPD minimizes the on-policy summed reverse KL
\begin{equation}\label{eq:opd}
    \mathcal{L}_{\mathrm{OPD}}(\theta) \triangleq \mathbb{E}_{x,\,y_{1:T} \sim \pi_\theta(\cdot \mid x)}\!\left[\sum_{t=1}^{T} \mathrm{KL}\!\left(\pi_\theta(\cdot \mid s_t) \,\big\|\, \pi_T(\cdot \mid s_t)\right)\right].
\end{equation}
Standard practice~\citep{agarwal2024policy,lu2025onpolicydistillation} treats the rollout as a stop-gradient sample and differentiates only the inner KL through the student logits, under which the contribution at the sampled token $y_t$ reduces to
\begin{equation}\label{eq:opd_grad}
    \nabla_\theta \mathcal{L}_{\mathrm{OPD}}(\theta)\big|_{y_t} = -\,\log \lambda_t^{\mathrm{OPD}} \cdot \nabla_\theta \log \pi_\theta(y_t \mid s_t), \qquad
    \log \lambda_t^{\mathrm{OPD}} \triangleq \log \pi_T(y_t \mid s_t) - \log \pi_\theta(y_t \mid s_t).
\end{equation}
Comparing Eq.~\eqref{eq:opd_grad} with Eq.~\eqref{eq:objective} shows that the OPD update is a REINFORCE estimator in which $\log\lambda_t^{\mathrm{OPD}}$ plays the role of a per-token reward rather than an advantage. The reward coincides with the KL-shaping term in PPO-based RLHF~\citep{schulman2017proximal} when the privileged teacher serves as the reference policy and the environment reward vanishes, so OPD is the special case of token-level PPO without a learned value baseline and without trust-region clipping. Every position is updated independently of its place in the trajectory, with the local signal $\log\lambda_t^{\mathrm{OPD}}$ decoupled from both the cumulative agreement along $y_{<t}$ and the terminal outcome $R$~\citep{yang2026self,song2026survey}.

\textbf{Advantage under the optimal policy.} The unbiasedness condition below Eq.~\eqref{eq:objective} leaves freedom in the reference policy under which $A_t$ is defined: replacing $A_t^{\pi_\theta}$ with the advantage under a policy $\pi^\ast$ that dominates $\pi_\theta$ in value preserves a valid policy-improvement direction~\citep{schulman2017proximal,jain2025learning}.
Let
\begin{equation}\label{eq:pi_star}
    \pi^\ast(y_t \mid s_t) \;\triangleq\; P(y_t \mid x, y_{<t}, R{=}1), \qquad V_t^{\pi^\ast}(s_t) \;\triangleq\; P^{\pi^\ast}(R{=}1 \mid s_t),
\end{equation}
denote the success-conditional optimal policy and its value, with $A_t^{\pi^\ast} \triangleq Q_t^{\pi^\ast}(s_t, y_t) - V_t^{\pi^\ast}(s_t)$ the corresponding one-step advantage. Equation~\eqref{eq:opd_grad} then admits a clean re-reading: the OPD teacher is an approximation $\pi_T \approx \pi^\ast$, so $\log\lambda^{\mathrm{OPD}}_t$ is the one-step expert log-ratio between $\pi^\ast$ and $\pi_\theta$, serving as a per-token surrogate for $A_t^{\pi^\ast}$ in the spirit of advantage-weighted regression.

\textbf{Bridge to our method.} So far we have identified two complementary failure modes of existing methods. The GRPO advantage in Eq.~\eqref{eq:grpo} resolves the credit budget $R - V_0$ uniformly across all $T$ tokens, while the OPD signal $\log\lambda_t$ in Eq.~\eqref{eq:opd_grad}, an expert log-ratio under $\pi^\ast$, is purely local and ignores both the cumulative evidence along $y_{<t}$ and the terminal outcome $R$.
A principled estimator should sit between the two extremes, retaining OPD's dense $\pi^\ast$-anchored signal while reweighting it by a position-dependent coefficient that tracks the running posterior of success encoded by Eq.~\eqref{eq:value}, recovering an approximation to $A_t/A_t^{\pi^\ast}$  without a learned value network.

\vspace{-2mm}
\section{Bayesian Evidence Aggregation}\label{sec:bayesian}
\vspace{-2mm}
The bridge of  Section~\ref{sec:opd} singled out $V_t$ as the right object for token-level credit assignment: the ideal advantage $A_t = V_{t+1} - V_t$ should telescope to $R - V_0$ in Eq.~\eqref{eq:ideal_adv}, so any accurate tracker of $V_t$ distributes the budget across tokens automatically but the challenge is to track $V_t$ without a learned critic. In this section, we first show that the problem reduces to learning an additive log-odds update driven by a per-token Bayes factor; Then we construct two tractable surrogate for the Bayes factors.

\vspace{-1mm}
\subsection{Bayesian Value Recursion}\label{sec:recursion}
\vspace{-1mm}
Applying Bayes' theorem to update the success probability after observing $y_t$ gives
\begin{equation}\label{eq:bayes_step}
    V_{t+1} \;=\; P^{\pi_\theta}(R{=}1 \mid x, y_{1:t}) \;=\; \frac{P(y_t \mid x, y_{<t}, R{=}1)\,V_t}{P(y_t \mid x, y_{<t})}.
\end{equation}
Define the per-token Bayes factor between the success and failure branches as
\begin{equation}\label{eq:lambda_star}
    \lambda_t^{\star} \;\triangleq\; \frac{P(y_t \mid x, y_{<t}, R{=}1)}{P(y_t \mid x, y_{<t}, R{=}0)} \;=\; \frac{P(y_t \mid x, y_{<t}, y_T{=}y^\ast)}{P(y_t \mid x, y_{<t}, y_T{\neq}y^\ast)},
\end{equation}
and let $\ell_t \triangleq \log\frac{V_t}{1-V_t}$ collect the running log-odds of success. 
The probability $P$ in Eqs.~\eqref{eq:bayes_step}--\eqref{eq:lambda_star} can be taken under either $\pi_\theta$ or $\pi^\ast$ from Eq.~\eqref{eq:pi_star}, yielding valid Bayesian recursions for $V_t^{\pi_\theta}$ and $V_t^{\pi^\ast}$ respectively, with the latter targeted by the expert log-ratio of Section~\ref{sec:opd}, though they are not exactly the same.
Using the law of total probability $P(y_t \mid x, y_{<t}) = P(y_t \mid x, y_{<t}, R{=}1)\,V_t + P(y_t \mid x, y_{<t}, R{=}0)\,(1-V_t)$, Eq.~\eqref{eq:bayes_step} can be rewritten as the additive log-odds recursion
\begin{equation}\label{eq:V_update_exact}
    V_{t+1} \;=\; \frac{\lambda_t^{\star}\,V_t}{\lambda_t^{\star}\,V_t + (1 - V_t)}, \qquad \ell_{t+1} \;=\; \ell_t + \log\lambda_t^{\star},
\end{equation}
in which every token contributes additively to $\ell_t$ via its log-Bayes-factor. Unrolling from $\ell_0 = \mathrm{logit}(V_0)$ and applying the sigmoid yields the closed-form value, the corresponding token-level advantage, and a telescoping credit identity, respectively,
\begin{equation}\label{eq:adv_exact}
    V_t \;=\; \sigma\!\left(\ell_0 + \sum_{s<t} \log\lambda_s^{\star}\right), \qquad A_t \;=\; \sigma(\ell_t + \log\lambda_t^{\star}) - \sigma(\ell_t),
\end{equation}
\begin{equation}\label{eq:budget_exact}
    \sum_{t=1}^{T} A_t \;=\; \sigma(\ell_{T+1}) - \sigma(\ell_0) \;=\; R - V_0.
\end{equation}
Two structural properties follow. Since $\sigma$ is Lipschitz with constant $\tfrac{1}{4}$, every token satisfies $|A_t| \leq \tfrac{1}{4}\,|\log\lambda_t^{\star}|$, and the universal range bound $|A_t| \leq 1$ applies regardless of $\log\lambda_t^{\star}$, so the per-token contribution is jointly bounded by $\min(\tfrac{1}{4}|\log\lambda_t^{\star}|, 1)$. Equation~\eqref{eq:budget_exact} fixes the total credit by the outcome and the prior alone, so any estimator that respects the additive log-odds structure of Eq.~\eqref{eq:V_update_exact} distributes the GRPO budget across tokens with no further bookkeeping.
Token-level credit assignment therefore reduces to estimating $\log\lambda_t^{\star}$, which is defined in terms of the success- and failure-conditional next-token distributions, which are not directly observable.

\vspace{-2mm}
\subsection{Two Oracle Surrogates for Bayesian Evidence}\label{sec:primitive}
\vspace{-1mm}
We construct estimators of $\lambda_t^{\star}$ under either probability measure ($\pi_\theta$ or $\pi^\ast$). Both arise from the same pair of modeling choices applied to whichever scorer $\hat\pi$ realizes the measure.

\begin{modelingchoice}[Oracle Conditioning]\label{mc:oracle}
The success-conditional distribution is approximated by appending $y^\ast$ to the prompt:
\begin{equation}\label{eq:oracle_approx}
    P(y_t \mid x, y_{<t}, R{=}1) \;\approx\; \hat\pi(y_t \mid x, y_{<t}, y^\ast).
\end{equation}
\end{modelingchoice}
The justification is the privileged-teacher rationale of Section~\ref{sec:opd}: conditioning on $y^\ast$ constrains the scorer to tokens consistent with reaching the correct answer.

\begin{modelingchoice}[Failure Branch]\label{mc:fail}
The failure-conditional distribution is approximated by the marginal under the same scorer:
\begin{equation}\label{eq:fail_approx}
    P(y_t \mid x, y_{<t}, R{=}0) \;\approx\; \hat\pi(y_t \mid x, y_{<t}).
\end{equation}
\end{modelingchoice}
The marginal admits the decomposition $\hat\pi(y_t \mid s_t) = P(y_t \mid s_t, R{=}1)\,V_t + P(y_t \mid s_t, R{=}0)\,(1-V_t)$: the failure branch dominates as $V_t \to 0$ and the approximation is tightest where incorrect trajectories spend most of their length, while the gap is largest near $V_t = \tfrac{1}{2}$ where the state weight $V_t(1-V_t)$ peaks.

Substituting Modeling Choices~\ref{mc:oracle} and~\ref{mc:fail} into Eq.~\eqref{eq:lambda_star} yields the unified estimator
\begin{equation}\label{eq:lambda_def}
    \lambda_t \;\triangleq\; \frac{\hat\pi(y_t \mid x, y_{<t}, y^\ast)}{\hat\pi(y_t \mid x, y_{<t})}, \qquad \log\lambda_t \;=\; \log\hat\pi(y_t \mid x, y_{<t}, y^\ast) - \log\hat\pi(y_t \mid x, y_{<t}),
\end{equation}
parameterized by a single scorer $\hat\pi$. Setting $\hat\pi = \pi_\theta$ yields the \emph{self-oracle}, which requires only one extra forward pass per trajectory and recovers the OPD per-token reward of Eq.~\eqref{eq:opd_grad} exactly, since the OPD privileged teacher is the student conditioned on $y^\ast$~\citep{zhao2026self}. 
Setting $\hat\pi = \pi_\phi$ for a frozen scorer that approximates $\pi^\ast$ at lower KL distance than $\pi_\theta$~\citep{agarwal2024policy,lu2025onpolicydistillation} yields the \emph{teacher-oracle}, which reduces approximation error at the cost of a separate scoring model; $\pi_\phi$ receives no gradient updates and need not be instruction-tuned, since its role is likelihood scoring rather than generation.

The pointwise gap between either estimator and the exact Bayes factor is captured by the success-branch KL
\begin{equation}\label{eq:eps_def}
    \epsilon_t \;\triangleq\; D_{\mathrm{KL}}\!\big(P(\cdot \mid x, y_{<t}, R{=}1) \,\big\|\, \hat\pi(\cdot \mid x, y_{<t}, y^\ast)\big),
\end{equation}
which traces a design axis: the self-oracle yields the largest $\epsilon_t$, and a perfect teacher drives $\epsilon_t \to 0$, recovering $\lambda_t = \lambda_t^{\star}$ and the exact recursion. Substituting $\lambda_t$ for $\lambda_t^{\star}$ in Eq.~\eqref{eq:V_update_exact} gives the implementable form, with the budget residual controlled by the cumulative oracle quality $\sum_t \epsilon_t$ (Appendix~\ref{app:bias}). The telescoping identity in Eq.~\eqref{eq:budget_exact} applies to the raw advantage before clipping, sign anchoring, and group normalization in Section~\ref{sec:method}; the implemented training advantage redistributes credit within each trajectory and inherits its direction from the verifier. Figure~\ref{fig:analysis}(c) stratifies the residual by length quartile and outcome: the mean rises monotonically from Q1 to Q4 in both classes, with $R{=}0$ above $R{=}1$ because $V_t$ saturates earlier on incorrect rollouts.

\begin{figure}[!t]
    \centering
    \includegraphics[width=\linewidth]{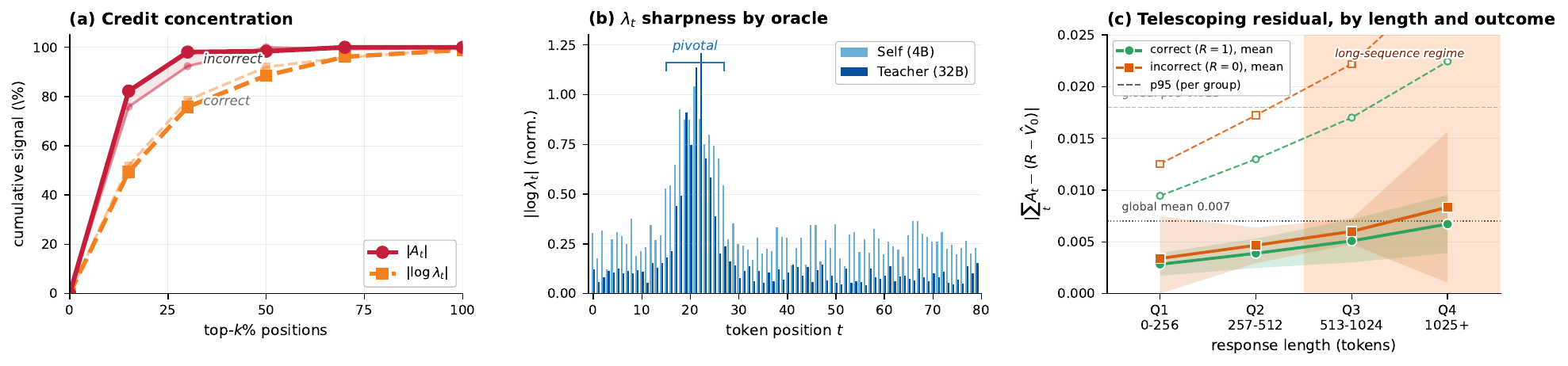}
    \caption{Empirical analysis on 200 MATH-500 trajectories from Qwen3-4B-Instruct. (a)~$|A_t|$ concentrates more sharply than $|\log\lambda_t|$, with the gap widest on incorrect trajectories where $V_t$ saturates early. (b)~The 32B teacher produces sharper $|\log\lambda_t|$ at pivotal tokens and lower background in determined regions. (c)~Telescoping budget residual stratified by length quartile and outcome. Solid lines mark the per-stratum mean of $|\sum_t A_t - (R - \hat V_0)|$ with 95\% confidence bands. The shaded region marks the long-sequence regime where evidence clipping caps the recursion.}
    \label{fig:analysis}
\end{figure}

\vspace{-2mm}
\subsection{Further Discussion and Empirical Validation}\label{sec:anatomy}
\vspace{-1mm}

Subtracting $V_t$ from both sides of Eq.~\eqref{eq:V_update_exact} gives the closed-form identity
\begin{equation}\label{eq:exact_adv}
    A_t \;=\; \frac{V_t(1-V_t)(\lambda_t - 1)}{\lambda_t \, V_t + 1 - V_t}.
\end{equation}
For $|\lambda_t - 1|$ small, $\lambda_t - 1 \approx \log\lambda_t$ and the denominator approaches $1$, yielding the first-order approximation
\begin{equation}\label{eq:first_order}
    A_t \;\approx\; V_t(1 - V_t) \cdot \log\lambda_t.
\end{equation}
The factorization separates two roles: the oracle evidence $\log\lambda_t$ provides per-token discrimination between success and failure branches, identical to the gradient signal of on-policy distillation in Eq.~\eqref{eq:opd_grad}, while the state weight $V_t(1-V_t)$ modulates the evidence by the running belief. Reliability diagrams in Appendix~\ref{app:value_accuracy} confirm that intermediate $V_t$ at $t = T/4, T/2, 3T/4$ is well calibrated.

Two predictions follow. Since $V_t(1-V_t)$ collapses once the running belief commits, credit concentrates on undetermined positions; on incorrect trajectories $V_t$ saturates toward $0$ early, so the state weight collapses even where $\log\lambda_t$ remains elevated by answer-correlated content. Figure~\ref{fig:analysis}(a) confirms the prediction: the cumulative signal of $|A_t|$ rises more steeply than that of $|\log\lambda_t|$, with the gap widest on incorrect trajectories. A stronger oracle should additionally sharpen $\lambda_t$ at pivotal positions while keeping the magnitude small elsewhere, and Figure~\ref{fig:analysis}(b) bears this out: the 32B teacher concentrates $|\log\lambda_t|$ more tightly around strategy-selection tokens than the 4B self-oracle. The prior $V_0$ supplies a complementary adaptation: $V_0(1-V_0)$ multiplies the entire trajectory, downweighting easy and unsolvable queries throughout, which yields prompt-difficulty filtering without manual thresholding~\citep{cui2025process}.

The state weight admits a variance-reduction interpretation under an idealized score-vector assumption. Writing the policy gradient as $\hat g = \sum_t w_t h_t$ with $h_t = \nabla_\theta \log\pi_\theta(y_t \mid x, y_{<t})$ and treating the $h_t$ as uncorrelated across positions, the variance gap between the state-blind weighting $w_t = \log\lambda_t$ and the Bayesian weighting $w_t = V_t(1-V_t)\log\lambda_t$ satisfies
\begin{equation}\label{eq:var}
\begin{split}
    \operatorname{Var}[\hat{g}_{0}] - \operatorname{Var}[\hat{g}_{1}] \;&\geq\; (1 - \gamma^2)\!\sum_{t \in \mathcal{D}} (\log\lambda_t)^2 \cdot \operatorname{Var}[h_t], \\
    \text{where} \quad \mathcal{D} \;&=\; \{\, t : |\log\lambda_t| \geq \delta,\; V_t(1 - V_t) < \gamma \,\}.
\end{split}
\end{equation}
Score vectors in autoregressive models are correlated across positions, so Eq.~\eqref{eq:var} is a directional indication rather than a quantitative guarantee, and the set $\mathcal{D}$ expands once the model commits to a solution path, predicting larger gains on longer trajectories that are consistent with the stratified residual in Figure~\ref{fig:analysis}(c) and the length-stratified accuracy in Table~\ref{tab:ext}(b).
\vspace{-3mm}
\section{\n: Oracle-Prompted Policy Optimization}\label{sec:method}
\vspace{-2mm}
The Bayesian value recursion of Section~\ref{sec:bayesian} is agnostic to how $\lambda_t$ is estimated. The remainder of the section specifies the estimator, defines the prior $V_0$, and the direction anchoring.

\textbf{Oracle estimator.} Computing $\lambda_t$ requires scoring every sampled token both with and without $y^\ast$ in context. In the \emph{self-oracle} mode, the policy $\pi_\theta$ scores both contexts, adding one forward pass per trajectory and no extra model. 
In the \emph{teacher-oracle} mode, a frozen $\pi_\phi$ replaces $\pi_\theta$ in both terms; $\pi_\phi$ receives no gradient updates and need not be instruction-tuned, since its role is likelihood scoring. 
The two modes operate at different points along the $\epsilon_t$ axis of Section~\ref{sec:primitive}, and once $\log\hat\lambda_{i,t}$ is computed every later step is identical, so the Bayesian machinery is decoupled from the choice of estimator.

\textbf{Prior and ratio computation.} Given $G$ trajectories per query, the prior is the $\operatorname{Beta}(\alpha,\alpha)$ posterior mean, and ratios are clipped for numerical stability:
\begin{equation}\label{eq:prior}
    \hat V_0 = \frac{k + \alpha}{G + 2\alpha}, \qquad \ell_0 = \log\frac{k + \alpha}{G - k + \alpha}, \qquad \log\hat\lambda_{i,t} = \operatorname{clip}\!\big(o_{i,t} - s_{i,t},\, -C,\, C\big),
\end{equation}
where $s_{i,t} = \log\hat\pi(y_{i,t} \mid x, y_{i,<t})$ and $o_{i,t} = \log\hat\pi(y_{i,t} \mid x, y^\ast, y_{i,<t})$. Setting $\alpha = 1$ keeps the log-odds finite in all-correct and all-incorrect groups, which would otherwise collapse every advantage to zero, while the clip bound $C$ prevents a single extreme $\log\lambda_t$ from saturating $\ell_t$ and silencing later positions. The log-odds then accumulate as $\ell_{i,t} = \ell_0 + \sum_{j<t}\log\hat\lambda_{i,j}$ and the raw Bayesian advantage is $A_{i,t}^{\mathrm{raw}} = \sigma(\ell_{i,t} + \log\hat\lambda_{i,t}) - \sigma(\ell_{i,t})$.

\textbf{Direction anchoring.} The sign of $A_{i,t}^{\mathrm{raw}}$ is determined by $\log\lambda_t$, not by the trajectory outcome. Any token with $\log\lambda_t > 0$ in an incorrect trajectory would receive positive advantage and reinforce a flawed reasoning step, and tokens in correct trajectories with $\log\lambda_t < 0$ would be penalized despite contributing to a successful answer; the failure mode mirrors pure self-distillation, where the gradient direction follows the privileged teacher rather than the environment reward~\citep{yang2026self}. To prevent the inversion, \n\ anchors the sign of every token-level advantage to the sequence-level GRPO advantage and normalizes across the group:
\begin{equation}\label{eq:anchored_adv}
    \hat{A}_{i,t} = \operatorname{sign}\!\big(A_i^{\mathrm{seq}}\big) \cdot |A_{i,t}^{\mathrm{raw}}|, \qquad A_i^{\mathrm{seq}} = \frac{R_i - \mu}{\sigma}, \qquad \tilde{A}_{i,t} = \frac{\hat{A}_{i,t} - \bar A}{\sigma_A + \epsilon}.
\end{equation}
The environment reward retains exclusive control over which trajectories are reinforced or penalized, while the Bayesian state weight and oracle evidence govern only the relative magnitude across positions within each trajectory.

\textbf{Policy update.} The normalized advantages enter the clipped surrogate objective
\begin{equation}\label{eq:ppo}
    \mathcal{L}(\theta) = \mathbb{E}\!\left[\frac{1}{G}\sum_{i=1}^{G} \frac{1}{T_i}\sum_{t=1}^{T_i} \min\!\left(\rho_{i,t}\,\tilde{A}_{i,t},\; \operatorname{clip}(\rho_{i,t},\, 1{-}\epsilon,\, 1{+}\epsilon)\,\tilde{A}_{i,t}\right)\right], \quad \rho_{i,t} = \frac{\pi_\theta(y_{i,t} \mid x, y_{i,<t})}{\pi_{\theta_{\mathrm{old}}}(y_{i,t} \mid x, y_{i,<t})}.
\end{equation}
Algorithm~\ref{alg:oppo} summarizes the procedure that replace the uniform sequence-level advantage $A_i^{\mathrm{seq}}$ by the token-level $\tilde{A}_{i,t}$, at the cost of one additional forward pass per trajectory.

\begin{wrapfigure}{r}{0.52\textwidth}
\vspace{-9mm}
\begin{minipage}{\linewidth}
\begin{algorithm}[H]
\caption{Oracle-Prompted Policy Optimization}\label{alg:oppo}
\begin{algorithmic}[1]
\Require Policy $\pi_\theta$, oracle estimator $\pi_{\mathrm{est}}$ ($\pi_\theta$ or $\pi_\phi$), group size $G$, prior $\alpha$, clip $C$
\For{each query $x$ with answer $y^\ast$}
    \State \stepcomment{\# Rollout and reward}
    \State Sample $\{y^{(i)}\}_{i=1}^G \sim \pi_\theta(\cdot \mid x)$
    \State Obtain $R_i$;\; $A_i^{\mathrm{seq}} \!\gets\! (R_i \!-\! \mu)/\sigma$
    \State $k \!\gets\! \sum_i R_i$;\; $\ell_0 \!\gets\! \log\frac{k+\alpha}{G-k+\alpha}$
    \State \stepcomment{\# Oracle evaluation}
    \For{$i = 1, \ldots, G$}
        \State $s_t \!\gets\! \log\pi_{\mathrm{est}}(y_{i,t} \mid x, y_{i,<t})$
        \State $o_t \!\gets\! \log\pi_{\mathrm{est}}(y_{i,t} \mid x, y^\ast, y_{i,<t})$
        \State $\log\hat{\lambda}_{i,t} \!\gets\! \operatorname{clip}(o_t \!-\! s_t, -C, C)$
    \EndFor
    \State \stepcomment{\# Bayesian credit assignment}
    \For{$i = 1, \ldots, G$}
        \State $\ell_{i,t} \!\gets\! \ell_0 + \sum_{j<t} \log\hat{\lambda}_{i,j}$
        \State $A_{i,t}^{\mathrm{raw}} \!\gets\! \sigma(\ell_{i,t} \!+\! \log\hat{\lambda}_{i,t}) - \sigma(\ell_{i,t})$
        \State $\hat{A}_{i,t} \!\gets\! \operatorname{sign}(A_i^{\mathrm{seq}}) \cdot |A_{i,t}^{\mathrm{raw}}|$
    \EndFor
    \State $\tilde{A}_{i,t} \!\gets\! (\hat{A}_{i,t} \!-\! \bar{A}) / (\sigma_A \!+\! \epsilon)$
\EndFor
\State \stepcomment{\# Policy update}
\State Update $\theta$ via Eq.~\eqref{eq:ppo}
\end{algorithmic}
\end{algorithm}
\end{minipage}
\vspace{-8mm}
\end{wrapfigure}

\vspace{-2mm}
\section{Results}\label{sec:exp}
\vspace{-2mm}

\begin{table}[t]
\centering
\caption{Pass@1 accuracy (\%) on reasoning benchmarks. All models are trained on DeepScaleR. \n~ variants are grouped by the underlying optimization framework and the oracle estimator. Subscript arrows denote improvement over the corresponding baseline.}
\label{tab:main}
\vspace{4pt}
\resizebox{\textwidth}{!}{
\begin{tabular}{l cccc cc c}
\toprule
 & \multicolumn{4}{c}{\textit{Mathematical Reasoning}} 
 & \multicolumn{2}{c}{\textit{Scientific Reasoning}} 
 & \textit{Code} \\
\cmidrule(lr){2-5} \cmidrule(lr){6-7} \cmidrule(lr){8-8}
Method & GSM8K & MATH-500 & AMC'23 & AIME'24 
       & GPQA-D & ARC-C & LCB \\
\midrule
\multicolumn{8}{l}{\textit{Qwen3-4B-Instruct-2507}} \\
SFT                      
  & 89.5 & 70.2 & 52.5 & 44.0 
  & 37.8 & 87.1 & 28.6 \\
GRPO             
  & 91.4 & 72.8 & 54.0 & 46.4 
  & 39.2 & 88.0 & 31.5 \\
Dr.GRPO          
  & 91.6 & 73.0 & 54.5 & 46.3 
  & 39.4 & 88.3 & 31.4 \\
DAPO             
  & 92.3 & 74.2 & 55.8 & 47.0 
  & 40.3 & 88.5 & 32.8 \\
SDPO             
  & 92.8 & 74.6 & 56.5 & 48.2 
  & 40.0 & 89.0 & 32.4 \\
\cmidrule(lr){1-8}
\multicolumn{8}{l}{\quad\textit{\small on GRPO}} \\
\oursrowa Self-\n      
  & 93.0\inc{1.6} & 76.0\inc{3.2} & 58.5\inc{4.5} 
  & 49.5\inc{3.1} & 41.6\inc{2.4} & 89.5\inc{1.5} 
  & 34.2\inc{2.7} \\
\oursrowb Teacher-\n   
  & 93.8\inc{2.4} & 77.4\inc{4.6} & 60.0\inc{6.0} 
  & 51.3\inc{4.9} & 42.5\inc{3.3} & 89.8\inc{1.8} 
  & 35.4\inc{3.9} \\
\multicolumn{8}{l}{\quad\textit{\small on DAPO}} \\
\oursrowa Self-\n     
  & 93.6\inc{1.3} & 76.8\inc{2.6} & 59.2\inc{3.4} 
  & 50.0\inc{3.0} & 41.8\inc{1.5} & 89.6\inc{1.1} 
  & 34.6\inc{1.8} \\
\oursrowb Teacher-\n  
  & \textbf{94.5}\inc{2.2} & \textbf{78.5}\inc{4.3} 
  & \textbf{61.5}\inc{5.7} 
  & \textbf{52.2}\inc{5.2} & \textbf{43.2}\inc{2.9} 
  & \textbf{90.0}\inc{1.5} 
  & \textbf{36.2}\inc{3.4} \\
\midrule
\multicolumn{8}{l}{\textit{Phi-4-mini-instruct}} \\
SFT         
  & 88.2 & 71.4 & 28.0 & 9.5 
  & 30.8 & 83.5 & 20.4 \\
GRPO             
  & 89.8 & 75.0 & 32.5 & 12.0 
  & 32.4 & 84.6 & 23.0 \\
Dr.GRPO          
  & 90.0 & 74.8 & 33.0 & 11.7 
  & 32.6 & 84.8 & 22.8 \\
DAPO             
  & 90.5 & 76.2 & 34.0 & 13.3 
  & 33.5 & 85.0 & 24.2 \\
SDPO             
  & 90.8 & 76.5 & 35.5 & 13.0 
  & 33.0 & 85.4 & 23.8 \\
\cmidrule(lr){1-8}
\multicolumn{8}{l}{\quad\textit{\small on GRPO}} \\
\oursrowa Self-\n      
  & 91.4\inc{1.6} & 78.2\inc{3.2} & 37.0\inc{4.5} 
  & 15.3\inc{3.3} & 34.5\inc{2.1} & 86.2\inc{1.6} 
  & 25.8\inc{2.8} \\
\oursrowb Teacher-\n   
  & 92.0\inc{2.2} & 80.0\inc{5.0} & 39.5\inc{7.0} 
  & 17.0\inc{5.0} & 35.8\inc{3.4} & 86.8\inc{2.2} 
  & 27.2\inc{4.2} \\
\multicolumn{8}{l}{\quad\textit{\small on DAPO}} \\
\oursrowa Self-\n     
  & 91.8\inc{1.3} & 78.8\inc{2.6} & 38.0\inc{4.0} 
  & 16.0\inc{2.7} & 35.0\inc{1.5} & 86.4\inc{1.4} 
  & 26.5\inc{2.3} \\
\oursrowb Teacher-\n  
  & \textbf{92.5}\inc{2.0} & \textbf{80.8}\inc{4.6} 
  & \textbf{40.5}\inc{6.5} 
  & \textbf{17.8}\inc{4.5} & \textbf{36.4}\inc{2.9} 
  & \textbf{87.0}\inc{2.0} 
  & \textbf{28.0}\inc{3.8} \\
\bottomrule
\end{tabular}
}
\vspace{-4mm}
\end{table}

\subsection{Experimental Setup}
\vspace{-1mm}
\textbf{Models and data.} We train all methods on DeepScaleR~\citep{deepscaler2025}, approximately 40K problem-answer pairs drawn from AMC, AIME, MATH, and OlympiadBench. Each problem comes with a ground-truth integer answer that serves both as the verifiable reward and as the oracle context $y^\ast$ for computing $\lambda_t$. We evaluate on two instruction-tuned models of comparable scale, Qwen3-4B-Instruct~\citep{yang2025qwen3} and Phi-4-mini-instruct~\citep{abdin2025phi}.

\textbf{Benchmarks.} For mathematical reasoning we use GSM8K~\citep{cobbe2021gsm8k}, MATH-500~\citep{lightman2023lets}, AMC~2023\footnote{\url{https://huggingface.co/datasets/AI-MO/aimo-validation-amc}}, and AIME~2024 \& 2025\footnote{\url{https://huggingface.co/datasets/AI-MO/aimo-validation-aime}}; for scientific reasoning, GPQA-Diamond~\citep{rein2024gpqa} and ARC-Challenge~\citep{allenai:arc}; for code generation, LiveCodeBench~\citep{jain2024livecodebench}. None of the benchmarks overlap the training data, and Pass@1 accuracy is averaged over 16 sampling runs.

\textbf{Comparison and configuration.} Baselines are drawn from the GRPO family: GRPO~\citep{guo2025deepseek}, Dr.GRPO~\citep{liu2025understanding}, DAPO~\citep{yu2025dapo}, and SDPO~\citep{hubotter2026reinforcement}. We additionally apply \n\ on top of DAPO to isolate whether the Bayesian advantage adds gains beyond optimization mechanics. The teacher-oracle uses Qwen3-32B as the estimator $\pi_\phi$ for both base models. Experiments run on $2 \times$ H200 GPUs using the verl framework~\citep{sheng2025hybridflow}, with AdamW at learning rate $1 \times 10^{-6}$, batch size $32$, group size $G = 8$, clip $\epsilon = 0.2$, evidence clip $C = 3.0$, and prior $\alpha = 1$.

\vspace{-3mm}
\subsection{Main Results}
\vspace{-2mm}

Table~\ref{tab:main} reports performance across all seven benchmarks. Three findings organize the discussion.

\textbf{Token-level credit improves reasoning.} Self-\n\ outperforms GRPO and its optimization variants on every benchmark for both base models, with gains concentrated on competition-level tasks where reasoning chains are longest. The pattern matches the variance argument of Section~\ref{sec:anatomy}: longer sequences accumulate more positions in determined regions where $V_t(1-V_t) \approx 0$, and the Bayesian advantage suppresses gradient signal at the saturated positions while GRPO distributes credit uniformly. The further gains over SDPO, which leverages oracle-conditioned signals without cumulative belief tracking, indicate that the running state estimate adds headroom beyond per-token discrimination.
\textbf{Oracle quality translates into training gains.} Teacher-\n\ improves over Self-\n\ on every benchmark, with the largest gains on the hardest tasks. Since the Bayesian machinery is identical in both variants, the improvement is attributable to sharper $\lambda_t$ estimates from the larger oracle. Scoring student tokens under a 32B model approximates a controlled distillation signal, so the additional gains over Self-\n\ should be read as evidence that the framework absorbs higher-quality $\lambda_t$ rather than as a within-model improvement at fixed scale; Self-\n\ remains the cleaner test of the Bayesian mechanism.
\textbf{The Bayesian advantage complements optimization improvements.} \n\ on top of DAPO yields gains comparable in magnitude to \n\ on GRPO, and Teacher-\n\ on DAPO reaches the highest performance on most benchmarks for both models. Improved clipping and filtering are therefore complementary to finer credit assignment rather than redundant with it.

\vspace{-2mm}
\subsection{Ablation and Analysis}
\vspace{-2mm}

Table~\ref{tab:ablation} removes each component of Self-\n\ individually on Qwen3-4B-Instruct. Direction anchoring is the most important: tokens in incorrect trajectories that correlate with $y^\ast$ produce $\log\lambda_t > 0$ and receive positive advantage, reinforcing flawed reasoning, and AIME'24 performance falls below GRPO without anchoring, mirroring the degradation of pure self-distillation~\citep{yang2026self}. State tracking ranks second, with the largest drops on AMC and AIME'24 where reasoning chains are longest: without the $V_t(1-V_t)$ factor, every position receives weight proportional to $|\log\lambda_t|$ regardless of trajectory commitment, and longer chains accumulate more saturated positions. The "$\log\lambda_t$-only" row, which removes both state tracking and clipping, sits close to GRPO across the four benchmarks; the gap to the full method quantifies the joint contribution of state tracking and numerical control, while the small margin over GRPO, negligible on AIME'24, isolates the value of per-token oracle evidence on its own.

\begin{wraptable}{l}{0.55\textwidth}
\vspace{-14pt}
\centering
\scriptsize
\caption{Component ablation on Qwen3-4B-Instruct-2507, concentrated on long-chain benchmarks.}\label{tab:ablation}
\vspace{2pt}
\begin{tabular}{l cccc}
\toprule
Configuration & M-500 & AMC & AIME'24 & GPQA \\
\midrule
\oursrowa Self-\n~ (full)        & 76.0 & 58.5 & 49.5 & 41.6 \\
\quad remove anchoring           & 71.8 & 51.0 & 43.0 & 38.5 \\
\quad remove tracking            & 74.5 & 55.8 & 46.8 & 40.0 \\
\quad remove clipping            & 75.4 & 57.5 & 48.6 & 41.2 \\
\quad remove prior               & 75.8 & 58.0 & 49.0 & 41.5 \\
\quad $\log\lambda_t$ only       & 73.6 & 55.0 & 46.5 & 39.8 \\
\midrule
GRPO (uniform)                   & 72.8 & 54.0 & 46.4 & 39.2 \\
\bottomrule
\end{tabular}
\vspace{3pt}
\includegraphics[width=1.0\linewidth]{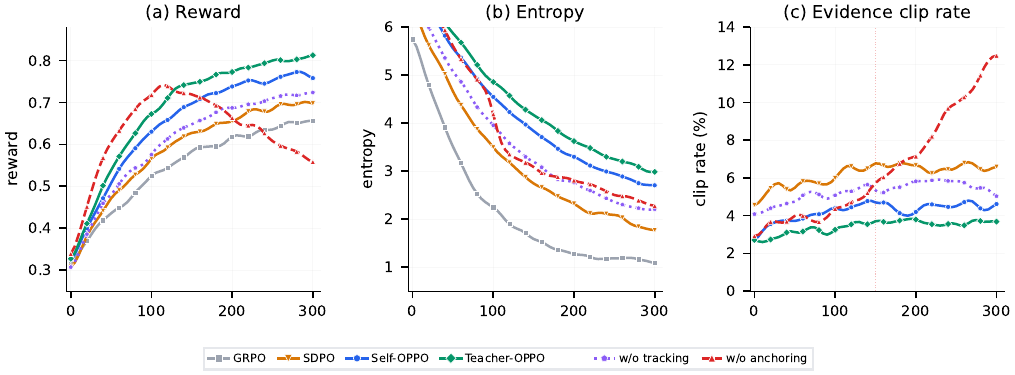}
\captionof{figure}{Without anchoring, reward collapses after step 150 as the clip rate diverges. Removing state tracking reduces the entropy preservation that separates \n\ from other methods.}
\label{fig:dynamics}
\vspace{-10pt}
\end{wraptable}

\begin{table}[!t]
\caption{(a)~Oracle estimator capacity: performance rises monotonically with teacher scale, with the largest jump at the 4B-to-8B transition and near-saturated gains beyond. All teachers are Qwen3 base models. (b)~\n\ gains widen with sequence length as more tokens enter determined regions.}\label{tab:ext}
\vspace{4pt}
\begin{minipage}[t]{0.44\textwidth}
\centering
\small
\textbf{(a)} Oracle estimator capacity
\vspace{4pt}

\begin{tabular}{l r cc}
\toprule
Estimator & Size & M-500 & AIME'24 \\
\midrule
\oursrowa Self   & 4B  & 76.0 & 49.5 \\
Qwen3-8B         & 8B  & 76.8 & 50.2 \\
Qwen3-14B        & 14B & 77.0 & 50.8 \\
\oursrowb Qwen3-32B & 32B & 77.4 & 51.3 \\
\bottomrule
\end{tabular}
\end{minipage}
\hfill
\begin{minipage}[t]{0.52\textwidth}
\centering
\small
\textbf{(b)} Accuracy by response length
\vspace{4pt}

\begin{tabular}{l cccc}
\toprule
Method & 0--256 & 257--512 
       & 513--1024 & 1025+ \\
\midrule
GRPO      & 88.5 & 80.2 & 68.4 & 52.0 \\
SDPO      & 89.0 & 81.5 & 70.8 & 55.2 \\
\oursrowa Self-\n 
          & 89.2 & 82.0 & 73.5 & 62.8 \\
\bottomrule
\end{tabular}
\end{minipage}
\vspace{-3mm}
\end{table}

The training dynamics in Figure~\ref{fig:dynamics} corroborate the static ablation. Reward ceiling and entropy floor both increase from GRPO through SDPO to Teacher-\n: GRPO applies uniform gradient pressure regardless of state and compresses the distribution at every position, whereas \n\ applies weak gradients where $V_t$ has saturated and preserves exploration capacity. The "w/o anchoring" run exhibits a distinctive failure mode: reward rises fastest in the early steps but declines after roughly step 150 as the evidence clip rate rises sharply and $\log\lambda_t$ values grow extreme on surface correlations with $y^\ast$, the runaway prevented by direction anchoring of the trajectory sign.

Two further analyses support the framework. Table~\ref{tab:ext}(a) scales the oracle while holding the student fixed: performance rises monotonically with teacher scale, with the largest jump at the 4B-to-8B transition and near-saturated gains beyond, confirming that the framework absorbs higher-quality $\lambda_t$ and that the estimator needs only next-token prediction capability rather than instruction-following ability. Table~\ref{tab:ext}(b) partitions MATH-500 by response length: all three methods perform comparably on the shortest bin because $V_t$ stays near $V_0$, while the gap widens with length and Self-\n\ shows the largest separation in the longest bin, matching the prediction of Section~\ref{sec:anatomy} that variance reduction grows with $|\mathcal{D}|$. Hyperparameter sensitivity and token-level credit visualizations are in Appendix~\ref{app:additional}.

\vspace{-3mm}
\section{Conclusion}
\vspace{-3mm}
\label{sec:conclusion}
\n\ recovers token-level advantages by accumulating oracle-conditioned likelihood ratios into a running Bayesian estimate of trajectory success, concentrating credit where the outcome remains uncertain and adding no learned components beyond one extra forward pass. Experiments across mathematical, scientific, and code reasoning benchmarks confirm consistent gains that widen on longer sequences, with stronger oracle estimators yielding monotonically better performance. Extending the framework beyond domains with verifiable answers remains an open direction, as does adapting the Bayesian recursion to richer outcome signals such as partial-credit rubrics.

\bibliographystyle{plainnat} 
\bibliography{ref}

\newpage
\appendix

\section{Extended Related Work}\label{app:related}

This appendix expands the discussion in Section~\ref{sec:intro} along three axes that frame \n\ within the broader landscape of credit assignment and reinforcement learning for reasoning. The body section establishes the immediate baselines; here we situate the framework against the wider literature.

\subsection{Trajectory-Level Optimization and Critic-Based Alternatives}\label{app:related_traj}

Group Relative Policy Optimization~\citep{guo2025deepseek} and its descendants form the dominant on-policy paradigm for reasoning. Dr.GRPO~\citep{liu2025understanding} fixes the reference policy in the importance ratio denominator to remove the length-bias of the original objective. DAPO~\citep{yu2025dapo} introduces asymmetric clipping and filtered sampling. GSPO~\citep{zheng2025group} aggregates probabilities at the sequence level. GMPO~\citep{zhao2025geometric} replaces the arithmetic-mean group baseline with a geometric mean. Across these variants, the per-token advantage remains a single trajectory-level scalar broadcast uniformly across positions~\citep{guo2025segment}, which is the limitation that motivates \n. The Bayesian recursion of Section~\ref{sec:bayesian} is orthogonal to the optimization-mechanic improvements, and Section~\ref{sec:exp} shows that applying \n\ on top of DAPO yields gains comparable in magnitude to applying it over standard GRPO.

The critic-based alternative is Proximal Policy Optimization~\citep{schulman2017proximal}, which pairs the clipped surrogate with a learned value network $V_\phi(s_t)$ that approximates the on-policy state value of Eq.~\eqref{eq:value} via temporal-difference regression. VinePPO\citep{kazemnejad2024vineppo} show that the learned critic in PPO barely outperforms a random baseline on long-horizon reasoning trajectories, where credit-assignment difficulty grows with sequence length and the regression target is dominated by sparse outcome rewards. The Bayesian value recursion of Section~\ref{sec:bayesian} can be read as an analytical alternative to fitting $V_\phi$: rather than learning the success probability through bootstrapped regression targets, the recursion computes it in closed form from the per-token oracle ratio, with the running posterior $V_t$ playing the role of $V_\phi$ at no parameter cost.

\subsection{Critic-Free Token-Level Credit and Distillation-Based Methods}\label{app:related_token}

Beyond GRPO and PPO, several critic-free approaches resolve credit at finer granularity than the trajectory level. VinePPO~\citep{kazemnejad2024vineppo} estimates per-step values through Monte Carlo rollouts that branch from each prefix; the estimator is unbiased but requires orders-of-magnitude more inference. TreeRL~\citep{hou2025treerl} extends the tree-search idea to multi-branch rollouts. PRIME~\citep{cui2025process} trains an implicit process reward model that supplies dense token-level rewards but requires online updating to avoid reward hacking. Segment Policy Optimization~\citep{guo2025segment} aggregates descendant outcomes within shared prefixes, sitting between trajectory-level and token-level granularity. Process reward models more broadly~\citep{lightman2023lets} train a separate verifier to score intermediate reasoning steps. Compared with all such methods, \n\ achieves token-level resolution at the cost of one extra forward pass per trajectory and without any auxiliary network or rollout, by exploiting the privileged $y^\ast$ context to derive a closed-form per-token signal. The trade-off is that \n\ requires verifiable training-time answers; methods that use Monte Carlo rollouts, learned critics, or step verifiers do not.

The line of work most closely related to \n\ uses on-policy distillation~\citep{agarwal2024policy,gu2023minillm,lu2025onpolicydistillation} with a privileged teacher. The reasoning-specific variant conditions the teacher on the ground-truth answer~\citep{zhao2026self,hubotter2026reinforcement}: RLSD~\citep{yang2026self} and SDPO~\citep{hubotter2026reinforcement} use the resulting answer-conditioned likelihood ratio as a token-level reward within GRPO. As established in Section~\ref{sec:opd}, the per-token signal $\log\lambda_t$ is the same quantity in all four works; the substantive distinction is what \n\ does with the signal. Where SDPO and RLSD apply $\log\lambda_t$ as an isolated per-token reward with uniform weight, \n\ accumulates it through the Bayesian recursion of Eq.~\eqref{eq:V_update_exact} and modulates each contribution by the running state weight $V_t(1-V_t)$. Section~\ref{sec:exp} shows that the additional Bayesian machinery yields consistent improvements over SDPO across all benchmarks, with the largest gains on long-chain competition-mathematics tasks where the state weight has the most room to redistribute credit.

\subsection{Methodological Connections}\label{app:related_method}

The recursion of Section~\ref{sec:recursion} is structurally a Bayesian filter over a binary latent variable (eventual success), with the per-token oracle ratio playing the role of the observation likelihood. The closed-form additivity in log-odds space is a standard property of Bayes' rule applied to binary outcomes~\citep{kalman1960new,jordan1999introduction}. 
The novelty is not the filter itself but its identification with token-level RL credit assignment: the running posterior $V_t$ coincides with the success probability targeted by the on-policy state value of Eq.~\eqref{eq:value}, and the one-step posterior change $V_{t+1} - V_t$ recovers the temporal-difference advantage $A_t = Q_t - V_t$ in the perfect-oracle limit. The connection to TD learning~\citep{sutton1988learning} is that \n\ provides an analytical, non-bootstrapped value estimator in a regime where bootstrapping is expensive (full reasoning rollout) and the target signal is sparse (binary terminal reward). Bayesian RL methods that maintain explicit posterior beliefs over reward or value functions~\citep{ghavamzadeh2015bayesian} typically require either conjugate priors or posterior-sampling machinery; the closed-form structure of the binary-outcome posterior here sidesteps both. The framework is therefore best understood as a Bayesian filtering interpretation of the existing on-policy distillation primitive, not as a fully Bayesian reinforcement-learning method.

A separate line of work uses learned or implicit value estimates to guide decoding at inference time, including verifier-augmented decoding~\citep{cobbe2021training} and value-guided sampling~\citep{khanov2024args}. \n\ operates entirely at training time: it modifies how trajectory-level reward is redistributed across tokens during gradient updates, and the deployed policy uses no oracle context, no verifier, and no tree search. The interpretation of $V_t$ as a running success probability does, however, suggest a natural inference-time analogue: the same recursion could in principle be applied with a self-oracle estimator at decoding time to weight token candidates by their belief contribution. Such a use of the framework is outside the scope of this paper and is left for future work. We note also that token-level credits learned via \n\ at training time may align with regions where verifier-guided decoding methods place high mass at inference, since both quantities track the model's running belief about eventual success; an empirical correlation analysis is similarly outside scope.

\section{Detailed Derivations}\label{app:derivations}

\subsection{Bayesian Update Derivation}\label{app:bayes_update}

The full derivation of the value recursion proceeds from the definition $V_t = P(R{=}1 \mid x, y_{1:t-1})$. Bayes' theorem gives the posterior after observing $y_t$:
\begin{equation}\label{eq:app_bayes_full}
    V_{t+1} = P(R{=}1 \mid x, y_{1:t}) = \frac{P(y_t \mid x, y_{<t}, R{=}1) \cdot V_t}{P(y_t \mid x, y_{<t})}.
\end{equation}
The denominator expands by the law of total probability,
\begin{equation}\label{eq:app_total_prob}
    P(y_t \mid x, y_{<t}) = P(y_t \mid x, y_{<t}, R{=}1)\,V_t + P(y_t \mid x, y_{<t}, R{=}0)\,(1 - V_t).
\end{equation}
Substituting Modeling Choice~\ref{mc:oracle} into the numerator and Modeling Choice~\ref{mc:fail} into the second denominator term, then dividing numerator and denominator by $\pi(y_t \mid x, y_{<t})$, yields
\begin{equation}\label{eq:app_recursion}
    V_{t+1} = \frac{\lambda_t \, V_t}{\lambda_t \, V_t + (1 - V_t)}, \qquad \lambda_t = \frac{\pi_{\mathrm{est}}(y_t \mid x, y_{<t}, y^\ast)}{\pi_{\mathrm{est}}(y_t \mid x, y_{<t})},
\end{equation}
which is exactly the same expression in the main text.

\subsection{Logit-Space Reduction}\label{app:logit}

Define the running log-odds $\ell_t = \log\frac{V_t}{1 - V_t}$. From Eq.~\eqref{eq:app_recursion},
\begin{equation}\label{eq:app_odds_ratio}
    \frac{V_{t+1}}{1 - V_{t+1}} = \frac{\lambda_t V_t / (\lambda_t V_t + 1 - V_t)}{(1 - V_t) / (\lambda_t V_t + 1 - V_t)} = \frac{\lambda_t V_t}{1 - V_t} = \lambda_t \cdot \frac{V_t}{1 - V_t},
\end{equation}
so taking logarithms collapses the multiplicative recursion in $V_t$ to additive accumulation in $\ell_t$:
\begin{equation}\label{eq:app_additive}
    \ell_{t+1} = \log\lambda_t + \ell_t, \qquad \ell_t = \ell_0 + \sum_{i=1}^{t-1} \log\lambda_i, \qquad V_t = \sigma(\ell_t) = \frac{1}{1 + e^{-\ell_t}}.
\end{equation}

\subsection{Advantage Closed Form}\label{app:adv_closed}

The token-level advantage $A_t = V_{t+1} - V_t$ admits a closed-form expression. Subtracting $V_t$ from Eq.~\eqref{eq:app_recursion} and clearing the common denominator,
\begin{equation}\label{eq:app_exact_adv}
    A_t = \frac{\lambda_t V_t - V_t(\lambda_t V_t + 1 - V_t)}{\lambda_t V_t + 1 - V_t} = \frac{\lambda_t V_t - \lambda_t V_t^2 - V_t + V_t^2}{\lambda_t V_t + 1 - V_t} = \frac{V_t(1 - V_t)(\lambda_t - 1)}{\lambda_t V_t + 1 - V_t}.
\end{equation}
The first-order approximation follows from two small-quantity expansions when $|\lambda_t - 1| \ll 1$. The denominator satisfies $\lambda_t V_t + 1 - V_t = 1 + (\lambda_t - 1) V_t \approx 1$, and the numerator satisfies $\lambda_t - 1 \approx \log\lambda_t$ via the Taylor expansion $e^x - 1 \approx x$. Substituting both gives
\begin{equation}\label{eq:app_first_order}
    A_t \;\approx\; V_t(1 - V_t) \cdot \log\lambda_t,
\end{equation}
which is Eq.~\eqref{eq:first_order} in the main text.

\subsection{Telescoping Property}\label{app:telescoping}

Using the identity $A_t = \sigma(\ell_{t+1}) - \sigma(\ell_t)$ established, the trajectory sum telescopes:
\begin{equation}\label{eq:app_telescope}
    \sum_{t=1}^{T} A_t = \sum_{t=1}^{T}\!\left[\sigma(\ell_{t+1}) - \sigma(\ell_t)\right] = \sigma(\ell_{T+1}) - \sigma(\ell_1).
\end{equation}
At the terminal state, the Bayesian posterior should converge to $V_{T+1} = R$, giving $\sigma(\ell_{T+1}) = R$ in the limit of an unconstrained recursion. With $\sigma(\ell_1) = V_0$ at the start, the budget identity becomes
\begin{equation}\label{eq:app_budget_identity}
    \sum_{t=1}^{T} A_t \;\approx\; R - V_0.
\end{equation}
Evidence clipping prevents $\ell_{T+1}$ from reaching $\pm\infty$ in practice, so $\sigma(\ell_{T+1})$ approximates rather than equals $R$. The empirical residual on 200 MATH-500 trajectories has mean $0.007$ and 95th percentile $0.018$ (Figure~\ref{fig:analysis}(c)), confirming that the deviation from the exact identity is small in aggregate.

\subsection{Lipschitz and Range Bounds on the Advantage}\label{app:lipschitz}
The sigmoid derivative satisfies $\sigma'(x) = \sigma(x)(1 - \sigma(x)) \leq \tfrac{1}{4}$ for all $x$, with equality at $x = 0$. The mean value theorem applied to $\sigma$ yields the Lipschitz bound, while the fact that $\sigma$ maps $\mathbb{R}$ into the open interval $(0, 1)$ yields a universal range bound:
\begin{equation}\label{eq:app_lipschitz}
    |A_t| = \big|\sigma(\ell_t + \log\lambda_t) - \sigma(\ell_t)\big| \;\leq\; \tfrac{1}{4}\,|\log\lambda_t|, \qquad |A_t| \;<\; 1 \text{ for any } \ell_t, \log\lambda_t \in \mathbb{R}.
\end{equation}
The two combine to give $|A_t| \leq \min\!\big(\tfrac{1}{4}|\log\lambda_t|, 1\big)$, with the Lipschitz bound active when $|\log\lambda_t| < 4$ and the range bound active otherwise. With evidence clipping $|\log\hat\lambda_t| \leq C$, the per-token contribution is further bounded by $|A_t| \leq \min(C/4, 1)$. For the manuscript's choice $C = 3.0$, the Lipschitz bound is binding and gives $|A_t| \leq 0.75$.

\section{Variance Reduction Proof}\label{app:variance}

We provide the full derivation of the bound in Eq.~\eqref{eq:var}, with the caveat already noted in Section~\ref{sec:anatomy}: the uncorrelated-score-vector assumption does not hold for autoregressive language models, so the bound should be read as a directional indication rather than a quantitative guarantee on practical training runs.

\begin{proposition}\label{prop:variance}
Let $\hat g = \sum_{t=1}^{T} w_t h_t$ denote the policy gradient estimator, where $h_t = \nabla_\theta \log\pi_\theta(y_t \mid x, y_{<t})$ is the score vector at position $t$. Suppose $\operatorname{Cov}[h_s, h_t] = 0$ for all $s \neq t$, and consider two weighting schemes,
\[
    w_t^{\mathrm{blind}} = \log\lambda_t, \qquad w_t^{\mathrm{Bayes}} = V_t(1 - V_t)\log\lambda_t.
\]
Then $\operatorname{Var}[\hat g^{\mathrm{blind}}] \geq \operatorname{Var}[\hat g^{\mathrm{Bayes}}]$, with the gap bounded below by
\begin{equation}\label{eq:app_var_bound}
    \operatorname{Var}[\hat g^{\mathrm{blind}}] - \operatorname{Var}[\hat g^{\mathrm{Bayes}}] \;\geq\; (1 - \gamma^2) \sum_{t \in T_{\mathrm{det}}} (\log\lambda_t)^2 \, \operatorname{Var}[h_t], \qquad T_{\mathrm{det}} = \{\, t : |\log\lambda_t| \geq \delta,\; V_t(1 - V_t) < \gamma \,\},
\end{equation}
for any thresholds $\delta, \gamma > 0$.
\end{proposition}

\begin{proof}
Under the uncorrelated-score assumption, the variance of each estimator decomposes additively across positions:
\begin{equation}\label{eq:app_var_decomp}
    \operatorname{Var}[\hat g^{\mathrm{blind}}] = \sum_{t=1}^{T} (\log\lambda_t)^2 \operatorname{Var}[h_t], \qquad \operatorname{Var}[\hat g^{\mathrm{Bayes}}] = \sum_{t=1}^{T} V_t^2(1 - V_t)^2 (\log\lambda_t)^2 \operatorname{Var}[h_t].
\end{equation}
The variance gap is therefore
\begin{equation}\label{eq:app_delta}
    \Delta = \operatorname{Var}[\hat g^{\mathrm{blind}}] - \operatorname{Var}[\hat g^{\mathrm{Bayes}}] = \sum_{t=1}^{T} \big[1 - V_t^2(1 - V_t)^2\big] (\log\lambda_t)^2 \operatorname{Var}[h_t].
\end{equation}
Since $V_t \in [0, 1]$ implies $V_t(1 - V_t) \leq \tfrac{1}{4}$, every factor satisfies $V_t^2(1 - V_t)^2 \leq \tfrac{1}{16}$, so $1 - V_t^2(1 - V_t)^2 \geq \tfrac{15}{16} > 0$ and $\Delta \geq 0$. To obtain the interpretable lower bound, restrict the sum to $T_{\mathrm{det}}$ and use $V_t^2(1 - V_t)^2 < \gamma^2$ on that subset:
\begin{equation}\label{eq:app_delta_lower}
    \Delta \;\geq\; \sum_{t \in T_{\mathrm{det}}} \big[1 - V_t^2(1 - V_t)^2\big] (\log\lambda_t)^2 \operatorname{Var}[h_t] \;\geq\; (1 - \gamma^2) \sum_{t \in T_{\mathrm{det}}} (\log\lambda_t)^2 \operatorname{Var}[h_t]. \qedhere
\end{equation}
\end{proof}

The bound is informative when $T_{\mathrm{det}}$ is large, which happens in two regimes. The first is long sequences, where $V_t$ saturates after a pivotal decision and many subsequent positions retain nonzero $\log\lambda_t$ from answer-correlated content. The second is easy or unsolvable queries, where $V_0$ already sits near $0$ or $1$ before any tokens are generated, so $V_t(1 - V_t) \approx 0$ throughout the trajectory while $\log\lambda_t$ remains nonzero. In both regimes the Bayesian state weight collapses to near zero on the saturating positions, so the corresponding gradient contributions are suppressed in $\hat g^{\mathrm{Bayes}}$ but not in $\hat g^{\mathrm{blind}}$. The empirical sequence-length stratification in Table~\ref{tab:ext}(b) is consistent with the predicted pattern: the gap between Self-\n\ and the state-blind baselines widens monotonically with response length.

\section{Oracle Quality Analysis}\label{app:oracle_quality}

\subsection{Bias Under Imperfect Oracle}\label{app:bias}

When $\epsilon_t > 0$, the estimated $\lambda_t$ differs from the true posterior ratio. The deviation has a clean interpretation. Let $p_t = P(y_t \mid x, y_{<t}, R{=}1)$ denote the true success-conditional probability and $q_t = \pi_{\mathrm{est}}(y_t \mid x, y_{<t}, y^\ast)$ the oracle estimate. Define the true ratio $\lambda_t^\ast$ and the estimated ratio $\hat\lambda_t$ on the same denominator:
\begin{equation}\label{eq:app_true_est_ratio}
    \lambda_t^\ast = \frac{p_t}{\pi(y_t \mid x, y_{<t})}, \qquad \hat\lambda_t = \frac{q_t}{\pi(y_t \mid x, y_{<t})}, \qquad \log\hat\lambda_t - \log\lambda_t^\ast = \log q_t - \log p_t.
\end{equation}
The KL definition $\epsilon_t = \mathbb{E}_{y_t \sim p_t}[\log p_t - \log q_t]$ gives the expected log-ratio bias under the true success-conditional distribution:
\begin{equation}\label{eq:app_bias}
    \mathbb{E}_{y_t \sim p_t}\!\left[\log\hat\lambda_t - \log\lambda_t^\ast\right] = -\,\epsilon_t.
\end{equation}
The bias is non-positive on average, so the oracle estimate is conservative and underestimates the true evidence for success. Stronger oracles reduce $\epsilon_t$ and the magnitude of the bias, which is consistent with the monotonic improvement in Table~\ref{tab:ext}(a).

\subsection{Why Same-Family Teachers Are Preferred}\label{app:same_family}

The oracle quality $\epsilon_t$ depends not only on model capacity but also on distributional alignment between the estimator and the true posterior. Cross-family teachers, for example a Llama model scoring Qwen trajectories, can produce large $\epsilon_t$ despite high capacity because tokenization, positional encoding, and pre-training corpus all differ. The ratio structure $\log\lambda_t = \log\pi_{\mathrm{est}}(y_t \mid x, y_{<t}, y^\ast) - \log\pi_{\mathrm{est}}(y_t \mid x, y_{<t})$ cancels many model-specific biases through subtraction, but residual mismatch in the conditional shape still inflates $\epsilon_t$. Same-family, larger-scale models are the recommended choice for the teacher-oracle role because the cancellation is most effective when the two scoring distributions share the same underlying parameterization.

\section{Connection to Distillation}\label{app:distillation}

\subsection{On-Policy Distillation as a Special Case}\label{app:opd_special}

Section~\ref{sec:opd} introduced the on-policy distillation objective $\mathcal{L}_{\mathrm{OPD}}$ with the oracle-conditioned teacher $\pi_T(\cdot \mid x, y_{<t}) \equiv \pi_{\mathrm{est}}(\cdot \mid x, y_{<t}, y^\ast)$, and showed that the per-token gradient at the sampled token $y_t$ reduces to a policy-gradient form with weight $\log\lambda_t$. Repeating the conclusion in the present notation,
\begin{equation}\label{eq:app_opd_grad}
    \nabla_\theta \mathcal{L}_{\mathrm{OPD}}(\theta)\big|_{y_t} = -\,\log\lambda_t \cdot \nabla_\theta \log\pi_\theta(y_t \mid x, y_{<t}), \qquad \log\lambda_t = \log\pi_{\mathrm{est}}(y_t \mid x, y_{<t}, y^\ast) - \log\pi_\theta(y_t \mid x, y_{<t}).
\end{equation}
Pure on-policy distillation is therefore a policy gradient with advantage $\log\lambda_t$, identical to the "$\log\lambda_t$ only" row of the ablation in Table~\ref{tab:ablation}. The Bayesian advantage in Eq.~\eqref{eq:first_order} differs from the distillation gradient in two specific ways. The state weight $V_t(1 - V_t)$ concentrates the gradient on positions where the outcome remains uncertain, so determined regions receive negligible updates. Direction anchoring binds the sign of every token-level advantage to the sequence-level GRPO advantage, so the environment reward retains exclusive control over which trajectories are reinforced or penalized. The two modifications are independent: removing either one degrades performance, but the failure mechanisms differ, as documented in the ablation analysis of Section~\ref{sec:exp}.

\subsection{Spectrum of Methods}\label{app:spectrum}

The factorization $A_t \approx V_t(1 - V_t) \cdot \log\lambda_t$ organizes a spectrum of methods along three axes: the per-token signal carried into the gradient, whether a state-tracking factor weights the signal, and which quantity controls the sign of each trajectory's update.

\begin{table}[h]
\centering
\caption{Method spectrum induced by the Bayesian factorization. Each row corresponds to one column setting in the ablation of Table~\ref{tab:ablation}: GRPO matches "uniform"; SDPO matches "$\log\lambda_t$ only" with environment-anchored direction; pure distillation (OPSD) matches "remove anchoring".}\label{tab:spectrum}
\small
\begin{tabular}{l c c c}
\toprule
Method & Per-token signal & State tracking & Direction \\
\midrule
GRPO~\citep{guo2025deepseek}                              & uniform $A_i^{\mathrm{seq}}$       & ---     & environment reward \\
Anchored distillation (SDPO)~\citep{hubotter2026reinforcement} & $|\log\lambda_t|$                  & ---     & environment reward \\
Self-\n\ (ours)                                            & $V_t(1{-}V_t)\,|\log\lambda_t|$    & Bayesian & environment reward \\
Teacher-\n\ (ours)                                         & $V_t(1{-}V_t)\,|\log\lambda_t|$, lower $\epsilon_t$ & Bayesian & environment reward \\
Pure distillation (OPSD)~\citep{zhao2026self}              & $\log\lambda_t$                    & ---     & teacher \\
\bottomrule
\end{tabular}
\end{table}

Reading the table from top to bottom moves through the design space. GRPO sits at the trajectory-level extreme, with no per-token discrimination and no state tracking. Anchored distillation introduces $\log\lambda_t$ as a per-token signal but applies it uniformly across positions. \n\ adds the Bayesian state weight $V_t(1 - V_t)$, modulating the per-token signal by the running belief about success. The self-oracle and teacher-oracle variants of \n\ differ only in the quality of the $\log\lambda_t$ estimate, parameterized by $\epsilon_t$. Pure distillation drops direction anchoring and lets $\log\lambda_t$ control the sign of the update, which is the failure mode reproduced in the "remove anchoring" row of Table~\ref{tab:ablation}.

\section{Extended Experimental Results}\label{app:extended}

\subsection{Full Results on Phi-4-mini}\label{app:phi4}

Table~\ref{tab:phi4_full} reports the complete results for Phi-4-mini across all benchmarks and ablation configurations, complementing the Qwen3-4B results in the main text. The pattern matches the Qwen3-4B case: \n\ improves over every baseline on every benchmark, with the largest absolute gains on AIME'24 and the largest relative gains on competition-level tasks where the base model is weakest.

\begin{table}[h]
\centering
\caption{Full results on Phi-4-mini (base). All methods are trained on DeepScaleR.}\label{tab:phi4_full}
\small
\begin{tabular}{l ccccccc}
\toprule
Method & GSM8K & M-500 & AMC'23 & AIME'24 & GPQA-D & ARC-C & LCB \\
\midrule
Base model         & 68.0 & 42.5 & 22.0 & 3.3  & 25.6 & 72.8 & 14.5 \\
+ GRPO             & 76.8 & 55.2 & 35.0 & 8.3  & 29.4 & 78.6 & 19.2 \\
+ Dr.GRPO          & 77.0 & 55.4 & 35.2 & 8.5  & 29.6 & 78.8 & 19.4 \\
+ DAPO             & 78.2 & 57.0 & 37.5 & 10.0 & 30.4 & 79.4 & 20.2 \\
+ SDPO             & 78.6 & 57.8 & 38.0 & 10.5 & 30.8 & 79.8 & 20.6 \\
\midrule
+ Self-\n          & 80.4 & 60.2 & 41.5 & 13.3 & 32.4 & 81.0 & 22.4 \\
+ Teacher-\n       & 81.5 & 62.0 & 43.5 & 15.0 & 33.6 & 81.6 & 23.6 \\
\midrule
+ Dr.GRPO + \n     & 80.6 & 59.8 & 40.8 & 12.8 & 32.0 & 80.8 & 22.0 \\
+ DAPO + \n        & 82.0 & 61.5 & 43.0 & 14.5 & 33.2 & 81.4 & 23.2 \\
\bottomrule
\end{tabular}
\end{table}

\subsection{Component Ablation on Phi-4-mini}\label{app:phi4_ablation}

The component ablation in Table~\ref{tab:phi4_ablation} repeats the analysis of Section~\ref{sec:exp} on Phi-4-mini. The ranking of components reproduces the Qwen3-4B result: direction anchoring is the most critical, followed by state tracking, with clipping and the prior contributing modestly. The "remove anchoring" row falls below GRPO on AIME'24, reproducing the inversion-failure pattern documented in the main text.

\begin{table}[h]
\centering
\caption{Component ablation on Phi-4-mini. The ranking of components matches the Qwen3-4B result, with direction anchoring the most critical and state tracking second.}\label{tab:phi4_ablation}
\small
\begin{tabular}{l cccc}
\toprule
Configuration & M-500 & AMC & AIME'24 & GPQA \\
\midrule
\oursrowa Self-\n\ (full)        & 60.2 & 41.5 & 13.3 & 32.4 \\
\quad remove anchoring           & 52.0 & 30.5 & 5.0  & 27.8 \\
\quad remove tracking            & 57.5 & 38.0 & 10.8 & 30.6 \\
\quad remove clipping            & 59.5 & 40.5 & 12.5 & 32.0 \\
\quad remove prior               & 59.8 & 41.0 & 13.0 & 32.2 \\
\quad $\log\lambda_t$ only       & 56.5 & 36.5 & 9.5 & 29.8 \\
\midrule
GRPO (uniform)                   & 55.2 & 35.0 & 8.3  & 29.4 \\
\bottomrule
\end{tabular}
\end{table}

\subsection{Pass@$K$ Results}\label{app:passk}

Table~\ref{tab:passk} reports Pass@$K$ for $K \in \{1, 4, 8, 16\}$ on MATH-500 and AIME'24, estimated from 32 samples per problem. \n\ improves Pass@$K$ at every sampling budget on both benchmarks, indicating that the gains over GRPO and SDPO are not driven by sharpening a single dominant mode at the expense of diversity. The relative improvement is largest at small $K$ on AIME'24, where the base policy spreads probability across many incorrect reasoning paths and the Bayesian advantage redistributes credit toward the few that succeed.

\begin{table}[h]
\centering
\caption{Pass@$K$ on Qwen3-4B. \n\ outperforms GRPO and SDPO at every sampling budget on both benchmarks.}\label{tab:passk}
\small
\begin{tabular}{l cccc cccc}
\toprule
 & \multicolumn{4}{c}{MATH-500} & \multicolumn{4}{c}{AIME'24} \\
\cmidrule(lr){2-5} \cmidrule(lr){6-9}
Method & @1 & @4 & @8 & @16 & @1 & @4 & @8 & @16 \\
\midrule
GRPO        & 72.6 & 82.0 & 86.5 & 90.2 & 26.7 & 40.0 & 48.3 & 56.7 \\
SDPO        & 74.8 & 83.5 & 87.8 & 91.0 & 29.0 & 43.3 & 51.7 & 60.0 \\
\oursrowa Self-\n     & 76.5 & 85.2 & 89.4 & 92.5 & 31.7 & 46.7 & 55.0 & 63.3 \\
\oursrowb Teacher-\n  & 78.0 & 86.5 & 90.2 & 93.0 & 33.3 & 48.3 & 56.7 & 65.0 \\
\bottomrule
\end{tabular}
\end{table}

\section{Hyperparameter Sensitivity}\label{app:hyperparams}

We study the sensitivity of Self-\n\ on Qwen3-4B to the three method-specific hyperparameters: evidence clip bound $C$, Beta prior strength $\alpha$, and group size $G$.

\subsection{Evidence Clip Bound $C$}\label{app:clip_sens}

The clip bound $C$ in Eq.~\eqref{eq:prior} controls the per-token contribution to the running log-odds. Table~\ref{tab:clip_sens} reports performance across five settings.

\begin{table}[h]
\centering
\caption{Sensitivity to evidence clip bound $C$ on Qwen3-4B.}\label{tab:clip_sens}
\small
\begin{tabular}{l ccccc}
\toprule
$C$ & 1.0 & 2.0 & 3.0 & 5.0 & 10.0 \\
\midrule
MATH-500 & 74.8 & 76.0 & 76.5 & 76.2 & 75.5 \\
AIME'24  & 29.0 & 31.0 & 31.7 & 31.3 & 30.0 \\
\bottomrule
\end{tabular}
\end{table}

Setting $C$ too small ($C = 1.0$) truncates genuine evidence at pivotal tokens, compressing the dynamic range of $A_t$ toward the uniform-advantage regime. Setting $C$ too large ($C = 10.0$) lets single tokens dominate the log-odds, saturating $V_t$ prematurely and silencing later positions. Performance is stable across $C \in [2.0, 5.0]$, and the choice $C = 3.0$ used in the main experiments sits at the center of the stable region.

\subsection{Beta Prior Strength $\alpha$}\label{app:alpha_sens}

The Beta prior in Eq.~\eqref{eq:prior} smooths the empirical group success rate before converting it to log-odds. Table~\ref{tab:alpha_sens} reports performance across five settings.

\begin{table}[h]
\centering
\caption{Sensitivity to Beta prior strength $\alpha$ on Qwen3-4B.}\label{tab:alpha_sens}
\small
\begin{tabular}{l ccccc}
\toprule
$\alpha$ & 0.1 & 0.5 & 1.0 & 2.0 & 5.0 \\
\midrule
MATH-500 & 76.0 & 76.3 & 76.5 & 76.4 & 75.8 \\
AIME'24  & 30.8 & 31.3 & 31.7 & 31.5 & 30.5 \\
\bottomrule
\end{tabular}
\end{table}

Small $\alpha$ makes $\hat V_0$ responsive to the empirical success rate but produces extreme initial log-odds in all-correct or all-incorrect groups. Large $\alpha$ pulls $\hat V_0$ toward $0.5$ regardless of the group composition and weakens the prompt-difficulty adaptation discussed in Section~\ref{sec:anatomy}. Performance is stable across $\alpha \in [0.5, 2.0]$, and $\alpha = 1$ corresponds to a Laplace prior that is the standard choice for a Bernoulli outcome.

\subsection{Group Size $G$}\label{app:group_sens}

The group size $G$ controls both the granularity of the prior $\hat V_0 = (k + \alpha) / (G + 2\alpha)$ and the variance of the GRPO sequence-level advantage that drives direction anchoring. Table~\ref{tab:group_sens} reports performance across four settings.

\begin{table}[h]
\centering
\caption{Sensitivity to group size $G$ on Qwen3-4B. \n\ improves over GRPO at every group size.}\label{tab:group_sens}
\small
\begin{tabular}{l cccc}
\toprule
$G$ & 4 & 8 & 16 & 32 \\
\midrule
GRPO       & 70.8 & 72.6 & 73.5 & 74.0 \\
\oursrowa Self-\n  & 74.5 & 76.5 & 77.2 & 77.5 \\
$\Delta$   & +3.7 & +3.9 & +3.7 & +3.5 \\
\bottomrule
\end{tabular}
\end{table}

Larger groups improve both GRPO and \n\ by reducing the variance of $\hat V_0$. The improvement of \n\ over GRPO is stable across group sizes, indicating that the token-level credit assignment provides gains complementary to the group-based prior estimation rather than depending on a particular group size.

\section{Additional Analysis}\label{app:additional}

\subsection{Value Estimation Accuracy}\label{app:value_accuracy}

The state-weight argument in Section~\ref{sec:anatomy} requires that intermediate $V_t$ track the evolving success probability, not merely that the terminal $V_{T+1}$ converge to $R$. We assess calibration on 500 MATH-500 trajectories from Qwen3-4B at four positions: $t = T/4, T/2, 3T/4, T$. Calibration at each position is measured by the Brier score
\begin{equation}\label{eq:app_brier}
    \mathrm{BS}(t) = \frac{1}{N}\sum_{i=1}^{N} \big(V_t^{(i)} - R^{(i)}\big)^2,
\end{equation}
which equals $0$ when the predicted success probability matches the empirical success rate at every decile.

\begin{table}[h]
\centering
\caption{Brier score of $V_t$ at four positions along the trajectory, on 500 MATH-500 trajectories. Lower is better.}\label{tab:brier}
\small
\begin{tabular}{l cccc}
\toprule
Oracle estimator & $T/4$ & $T/2$ & $3T/4$ & $T$ \\
\midrule
\oursrowa Self (4B)     & 0.071 & 0.058 & 0.049 & 0.042 \\
\oursrowb Teacher (32B) & 0.054 & 0.043 & 0.034 & 0.028 \\
\bottomrule
\end{tabular}
\end{table}

The Brier score decreases monotonically along the trajectory under both estimators, consistent with the recursion sharpening the posterior as more evidence is observed. 
The curves stay close to the diagonal at every position for both estimators, with the teacher curves consistently tighter in the mid-range $V_t \in [0.3, 0.7]$ where the state weight $V_t(1-V_t)$ has the largest effect on the gradient. The intermediate calibration confirms that $V_t \approx \tfrac{1}{2}$ corresponds to roughly $50\%$ empirical success, so the state-weight modulation peaks where the outcome is genuinely uncertain rather than at an artefact of the recursion.


\subsection{Computational Overhead}\label{app:overhead}

\n\ requires one additional forward pass per trajectory for oracle evaluation. Table~\ref{tab:overhead} reports wall-clock training time per step on $2 \times$ H200 GPUs.

\begin{table}[h]
\centering
\caption{Wall-clock time per training step on $2 \times$ H200 GPUs.}\label{tab:overhead}
\small
\begin{tabular}{l ccc}
\toprule
Method & Forward passes & Time/step (s) & Overhead vs. GRPO \\
\midrule
GRPO                   & 1 (rollout) + 1 (policy)                   & 12.4 & --- \\
\oursrowa Self-\n      & 1 (rollout) + 1 (oracle) + 1 (policy)      & 16.8 & $+35\%$ \\
\oursrowb Teacher-\n   & 1 (rollout) + 1 (oracle, 32B) + 1 (policy) & 24.2 & $+95\%$ \\
\bottomrule
\end{tabular}
\end{table}

The self-oracle overhead is moderate ($+35\%$) because the oracle forward pass uses the same model as the policy and can share the KV cache from the standard forward pass. The teacher-oracle overhead is larger ($+95\%$) because the 32B estimator runs on separate weights, but the cost can be amortized by batching oracle evaluations across trajectories within a step and by offloading the estimator to dedicated devices when available.

\section{Implementation Details}\label{app:implementation}

\subsection{Training Configuration}\label{app:train_cfg}

All experiments use the verl framework~\citep{sheng2025hybridflow} on $2 \times$ H200 GPUs with the configuration in Table~\ref{tab:hparams}.

\begin{table}[h]
\centering
\caption{Training hyperparameters used for all main-text experiments.}\label{tab:hparams}
\small
\begin{tabular}{ll}
\toprule
Hyperparameter & Value \\
\midrule
Optimizer & AdamW \\
Learning rate & $1 \times 10^{-6}$ \\
Weight decay & $0.01$ \\
Warmup steps & $10$ \\
Learning-rate schedule & cosine decay \\
Global batch size & $32$ \\
Group size $G$ & $8$ \\
Maximum response length & $8192$ tokens \\
Maximum prompt length & $1024$ tokens \\
Sampling temperature & $1.0$ \\
PPO clip $\epsilon$ & $0.2$ \\
Evidence clip $C$ & $3.0$ \\
Beta prior $\alpha$ & $1$ \\
Training epochs & $1$ \\
Gradient checkpointing & enabled \\
Precision & bfloat16 \\
\bottomrule
\end{tabular}
\end{table}

\subsection{Oracle Prompt Format}\label{app:oracle_prompt}

For self-oracle evaluation, the ground-truth answer $y^\ast$ is appended to the prompt before the model's response, in the format
\begin{verbatim}
[Standard prompt]
Question: {question}
[Oracle augmentation]
The answer is {y*}.
[Model response]
{y_1, y_2, ..., y_T}
\end{verbatim}
The oracle-conditioned forward pass scores the same response tokens $y_{1:T}$ under the augmented prompt, while the standard forward pass scores them under the original prompt. The log-likelihood ratio $\log\lambda_t$ is the per-token difference between the two scores. No generation is performed during oracle evaluation, so the cost reduces to two forward passes per trajectory rather than two decoding runs.

\subsection{Reward Function}\label{app:reward}

We use exact-match reward: $R = 1$ if the extracted numerical answer matches the ground-truth integer and $R = 0$ otherwise. Answer extraction follows the standard boxed-answer format used in DeepScaleR. No partial-credit or format-bonus rewards are used, so the binary outcome reward is the only environment signal driving direction anchoring.


\end{document}